\documentclass[pdflatex,sn-mathphys,a4paper]{sn-jnl}

\DeclareMathOperator{\trace}{trace}

\jyear{2022}%

\theoremstyle{thmstyleone}%
\newtheorem{theorem}{Theorem}
\newtheorem{proposition}[theorem]{Proposition}%
\newtheorem{lemma}[theorem]{Lemma}
\newtheorem{corollary}[theorem]{Corollary}

\theoremstyle{thmstyletwo}%
\newtheorem{remark}{Remark}%

\theoremstyle{thmstylethree}%

\begin{document}

\title{Test Set Sizing Via Random Matrix Theory}


\author[1]{\fnm{Alexander} \sur{Dubbs} \dgr{PhD}}\email{alex.dubbs@gmail.com}
\affil[1]{\orgaddress{\street{400 Central Park
      West}, \city{New York}, \state{NY} \postcode{10025}, \country{USA}}}


\abstract{This paper uses techniques from Random Matrix Theory to find the ideal
training-testing data split for a simple linear regression with $m$
data points, each an independent $n$-dimensional multivariate
Gaussian. It defines ``ideal'' as satisfying the {\it integrity
  metric}, i.e. the empirical model error is the actual measurement
noise, and thus fairly reflects the value or lack of same of the
model. This paper is the first to solve for the training and test set sizes
for any model in a way that is truly optimal. The number of data
points in the training set is the root of a quartic polynomial
Theorem~\ref{thm:vanillaregression} derives which depends only on $m$
and $n$; the covariance matrix of the multivariate Gaussian, the true
model parameters, and the true measurement noise drop out of the
calculations. The critical mathematical difficulties were realizing
that the problems herein were discussed in the context of the Jacobi
Ensemble, a probability distribution describing the eigenvalues of a
known random matrix model, and evaluating a new integral in the style
of Selberg and Aomoto. Mathematical results are supported with
thorough computational evidence. This paper is a step towards
automatic choices of training/test set sizes in machine learning.}


\maketitle

\section{Introduction}

The problem of determining a training and test set is almost as old as machine learning, and its importance is clear to any practitioner.  Furthermore, the problem of finding the optimal training-testing division has never been solved analytically or asymptotically for any model using an entirely correct metric of success and parameters available to the modeler.  Typically the sizes of training and test sets are determined by gut and tradition.  This paper takes a substantial step towards changing this state of affairs. It measures success by determining how well the test set error matches the actual error in measurement, i.e. the training-testing division maximizes {\it integrity}. This metric describes how useful a model would be in a real business setting, the extent to which it should be relied upon or justifiably not relied upon. It should be used for any model for which it is available - a major result of this paper is that {\it integrity} can be maximized without the knowledge of the true error in measurement in the context of a linear model (or the covariance of the data or the true model parameters). This paper finds the optimal training/testing split for a linear regression, assuming data points are i.i.d. from a Gaussian with some covariance and zero mean. This model can be analyzed using Random Matrix Theory, and studying it leads to the surprising result that the number of training points should be $O(m^{2/3})$ if $m \gg n$, where $m$ is the number of points in the data set and $n$ is their dimension. The leading order term is $(2n+n^2)^{1/3}m^{2/3}$, although convergence is slow (this assumes $n$ is small and fixed), a better approximation is found in Corollary~\ref{cor:asymptotics}. For $m$ nearer to $n$ the training set size can be found by solving a simple quartic equation described in Theorem~\ref{thm:vanillaregression}, which in the case that $m$ is near $n$ can have surprising results. In the medium-data case where $m$ is not near $n$ but not enormous, this analysis confirms the conventional wisdom that using approximately half of the data for training is best. Assuming prior distributions are picked appropriately, the amount of data required for modeling should only decrease, so the $O(m^{2/3})$ leading order term can be interpreted as an upper bound. Since it is easy to solve a quartic equation numerically, it is possible that finding the ideal {\it a priori} training, validation, and test set sizes for a given model may soon be possible to do quickly and accurately. Future work will attempt remove the normality assumptions of random variables.

This problem has been studied in the literature in both discrete and continuous contexts. Some authors have found more general results, but none are as philosophically appealing as those that maximize the {\it integrity metric}. \cite{Larsen1999} finds that the training set grows as $O(m^{2/3})$, as do we, using the integrity metric on a simpler problem. \cite{Picard1990} and \cite{Afendras2019} find that the training set should grow as $O(m)$ for different metrics than this paper's, in the former using a linear regression and in the latter using a much more general setup. This paper's contribution is to find a training-testing split that results in a loss that is what it theoretically {\it should} be. \cite{Guyon1997} and \cite{Guyon1998} look at the question of determining the test set size in the case of discrete labels (this paper's are continuous) and try to minimize test-set error. The former also finds an $O(m)$ answer for the test set size. Both versions require tuning parameters that specify accuracy needs {\it a priori}, this work improves on theirs by requiring no such parameters and finding a training-testing split that provides an error estimate that is again what it {\it should} be. \cite{Kearns1997} studies the training-test split question using the VC-Dimension, also with the intent of maximizing test-set error on a discrete problem, and also finds a very different answer than this paper's, also reliant on a tuning parameter which describes the complexity of the target function to be learned.

The three major ensembles, or probability distribution functions, in finite dimensional Random Matrix Theory, are the Hermite, Laguerre, and Jacobi Ensembles, each of which is the eigenvalue distribution of a certain simple random matrix. There are versions of them for random matrices over the reals, complex numbers, and quaternions, and they have even been generalized further; \cite{Dumitriu2002} provides a clear introduction. This paper focuses on the real Jacobi Ensemble, since certain known integrals over it (Selberg \cite{Selberg1944}, and Aomoto \cite{Aomoto1987}), and one new one proven in this paper provide the traces of random matrices that arise in the forthcoming mathematical development.  They are used to prove the following theorem:
\newcounter{vanillaregression}
\setcounter{vanillaregression}{\value{theorem}}
\newcounter{vanillaregressionsection}
\setcounter{vanillaregressionsection}{\value{section}}
\begin{theorem}\label{thm:vanillaregression} For a plain vanilla linear regression with $m$ data points each assumed to be $n$-dimensional Gaussians with some covariance and mean zero, the number of data points $p$ that should be in the training set satisfies:
\[
  \begin{split}\delta_{m,n}(p) &=
    9 + n (m^2 (2 + n)^2 + 3 (4 + n) - 2 m (5 + 2n))\\
   &\qquad{} - 24 p -  n (12 + m^2 (2 + n) + 2 m n (3 + n)) p \\
  &\qquad{}+(22 + n (8 + 2 m (1 + n) + n (3 + n))) p^2 \\
  &\qquad{}- (8 + n (2 + n)) p^3 + p^4 \\
  &=0,\end{split}
\]
rounded to the nearest integer, giving $p = O(m^{2/3})$ for $m \gg n$, with leading order term $(2n+n^2)^{1/3}m^{2/3}$. $p$ is the best if the expected test set error is nearly the true error in measurement, i.e. the solution has the most {\it integrity}. Note that to find $p$, the true covariance matrix of the Gaussians, the true model parameters, and the true error in measurement do not need to be known.
\end{theorem}

A more precise statement is given in subsection~\ref{subsec:PrimaryResults}.

\section{Results}

\subsection{Background from Random Matrix Theory}

Let 
\[w_{n,\alpha,\beta,\gamma}(x) =
\prod_{i<j}\vert{}x_i-x_j\rvert{}^{2\gamma}\prod_{i=1}^{n}x_i^{\alpha-1}\prod_{i=1}^{n}(1-x_i)^{\beta-1}dx.\]
Let
\[
\left\langle f(x_1,\ldots,x_n)\right\rangle_{n,\alpha,\beta,\gamma} =
\frac{\int_0^{1}\cdots\int_{0}^{1}f(x_1,\ldots,x_n)w_{n,\alpha,\beta,\gamma}(x)dx}{\int_0^{1}\cdots\int_{0}^{1}w_{n,\alpha,\beta,\gamma}(x)dx}.
\]

\begin{proposition}\label{prop:Selberg}
Selberg's Integral states \cite{Selberg1944}:
\[ S_n(\alpha,\beta,\gamma)=\int_0^{1}\cdots\int_{0}^{1}w_{n,\alpha,\beta,\gamma}(x)dx =\prod_{i=0}^{n-1}\frac{\Gamma(\alpha+i\gamma)\Gamma(\beta +i\gamma)\Gamma(1+(i+1)\gamma)}{\Gamma(\alpha+\beta+(n+i-1)\gamma)\Gamma(1+\gamma)}.
\]
\end{proposition}

\begin{proposition}\label{prop:Aomoto}
Aomoto's Integral states for $k < n$ \cite{Aomoto1987}:
\[ \left\langle \prod_{i=1}^{k}x_i\right\rangle_{n,\alpha,\beta,\gamma} = \prod_{i=1}^{k}\frac{\alpha + (n-i)\gamma}{\alpha+\beta+(2n-i-1)\gamma}. \]
\end{proposition}

These integral identities are relevant to finding the eigenvalue
distribution of several ensembles of random matrices.
\begin{proposition}\label{prop:Jacobimodel}
  The following matrix model:
  \[X^tX/(X^tX + Y^tY)^{-1},\]
  has eigenvalue probability distribution function equal to
  $w_{n,\alpha,\beta,\gamma}(x)/S_n(\alpha,\beta,\gamma)$, also called
  the {\it Jacobi Ensemble}, where $X$ is $p\times n$ full of
  i.i.d. Gaussians, $Y$ is $(m-p)\times n$ full of i.i.d. Gaussians,
  $x$ is the list of $n$ eigenvalues, $\alpha = \frac{1}{2}\left(p - n
  + 1\right)$, $\beta = \frac{1}{2}\left(m - p - n + 1\right)$, and
  $\gamma = 1/2$ (Different values of $\gamma$ are used for Gaussian
  random matrices over different algebras, the complex numbers and the
  quaternions. It is retained to preserve generality. $2\gamma$ is
  sometimes called $\beta$ in the literature, and is the dimension of
  the algebra). $X^tX/(X^tX + Y^tY)^{-1}$ is the {\it matrix model for
    the Jacobi Ensemble}. \cite{Dumitriu2002} provides a clear
  introduction to the Jacobi Ensemble and its matrix models in
  context.
\end{proposition}

Next, Aomoto's result is generalized to calculate more moments of the Jacobi Ensemble. Following the logic of Aomoto's proof as described in \cite{Andrews1999}:

\begin{theorem}\label{thm:Jacobimoments}
\[
    \left\langle x_1^{-2}\right\rangle_{n,\alpha,\beta,\gamma} =
    \frac{\alpha+\beta+(n-1)\gamma-1}{(\alpha-1)(\alpha-2)}\\
    \cdot\left(\alpha+\beta - 2 + \frac{\gamma(n-1)(\alpha+\beta+n\gamma-1)}{\alpha+\gamma -1}\right).  \]
\end{theorem}
\begin{proof}
\begin{align*}
0&=\int_{0}^{1}\cdots\int_{0}^{1}\frac{\partial}{\partial x_1}\left(x_1^ax_2^2\cdots x_n^2w_{n,\alpha,\beta,\gamma}(x)\right)dx \\
&= (a + \alpha - 1)\left\langle x_1^{a-1}x_2^2\cdots
x_n^2\right\rangle_{n,\alpha,\beta,\gamma} \\
&\qquad{}- (\beta - 1)\left\langle\frac{x_1^ax_2^2\cdots x_n^2}{1-x_1}\right\rangle_{n,\alpha,\beta\,\gamma}
 + 2\gamma\sum_{j=2}^{n}\left\langle\frac{x_1^ax_2^2\cdots x_n^2}{x_1-x_j}\right\rangle_{n,\alpha,\beta,\gamma}
\end{align*}

By substituting $a = 1$, we get
\[
0=\alpha\left\langle x_2^2\cdots
x_n^2\right\rangle_{n,\alpha,\beta,\gamma} - (\beta -
1)\left\langle\frac{x_1x_2^2\cdots
  x_n^2}{1-x_1}\right\rangle_{n,\alpha,\beta\,\gamma}
+ 2\gamma\sum_{j=2}^{n}\left\langle\frac{x_1x_2^2\cdots x_n^2}{x_1-x_j}\right\rangle_{n,\alpha,\beta,\gamma}
\]
Similary, by substituting $a = 2$, we get
\[
0=(\alpha + 1)\left\langle x_1x_2^2\cdots
x_n^2\right\rangle_{n,\alpha,\beta,\gamma} 
- (\beta - 1)\left\langle\frac{x_1^2x_2^2\cdots x_n^2}{1-x_1}\right\rangle_{n,\alpha,\beta\,\gamma} + 2\gamma\sum_{j=2}^{n}\left\langle\frac{x_1^2x_2^2\cdots x_n^2}{x_1-x_j}\right\rangle_{n,\alpha,\beta,\gamma}
\]

Let us look at the rightmost averages in the two previous equations. The latter changes sign if the indices $1$ and $j$ are swapped, so it is zero. To get the former, notice that:
\[ \frac{x_1x_j^2}{x_1-x_j} + \frac{x_jx_1^2}{x_j-x_1} = -x_1x_j, \]
so the average becomes $-\frac{1}{2}\left\langle x_1x_2x_3^2\cdots x_n^2\right\rangle_{n,\alpha,\beta,\gamma}$. Also note that by factorization and cancelling:
\[
  \left\langle\frac{x_1x_2^2\cdots
    x_n^2}{1-x_1}\right\rangle_{n,\alpha,\beta\,\gamma}  -
  \left\langle\frac{x_1^2x_2^2\cdots
    x_n^2}{1-x_1}\right\rangle_{n,\alpha,\beta\,\gamma}
  = \left\langle x_1x_2^2\cdots
  x_n^2\right\rangle_{n,\alpha,\beta,\gamma}
\]

So subtracting the $a = 2$ equation from the $a = 1$ equation,
\[
0 = \alpha\left\langle x_2^2\cdots x_n^2\right\rangle_{n,\alpha,\beta,\gamma} - (\alpha + \beta)\left\langle x_1x_2^2\cdots x_n^2\right\rangle_{n,\alpha,\beta,\gamma} -(n-1)\gamma\left\langle x_1x_2x_3^2\cdots x_n^2\right\rangle_{n,\alpha,\beta,\gamma},
\]
or, equivalently,

\begin{multline*}
  0 = \alpha\left\langle x_1^2\cdots x_{n-1}^2\right\rangle_{n,\alpha,\beta,\gamma} 
 - (\alpha + \beta)\left\langle x_1\cdots
 x_{n-1}\right\rangle_{n,\alpha+1,\beta,\gamma}\times\frac{S_n(\alpha+1,\beta,\gamma)}{S_n(\alpha,\beta,\gamma)} \\
- (n-1)\gamma\left\langle x_1\cdots x_{n-2}\right\rangle_{n,\alpha+1,\beta,\gamma}\times\frac{S_n(\alpha+1,\beta,\gamma)}{S_n(\alpha,\beta,\gamma)}.
\end{multline*}

Now using Proposition~\ref{prop:Selberg} and~\ref{prop:Aomoto},
\begin{multline*}
\left\langle x_1\cdots
x_k\right\rangle_{n,\alpha+1,\beta,\gamma}\cdot\frac{S_n(\alpha+1,\beta,\gamma)}{S_n(\alpha+2,\beta,\gamma)}\\
\begin{aligned}
&=\prod_{i=1}^{k}\frac{\alpha+1+(n-i)\gamma}{\alpha+\beta+1+(2n-i-1)\gamma}
\times\prod_{i=0}^{n-1}\frac{\alpha+\beta+1+(n+i-1)\gamma}{\alpha+1+i\gamma} \\
&=
\prod_{i=1}^{k}\frac{\alpha+1+(n-i)\gamma}{\alpha+\beta+1+(2n-i-1)\gamma}
\times\prod_{i=1}^{n}\frac{\alpha + \beta + 1+(2n - i - 1)\gamma}{\alpha+1+(n-i)\gamma}\\
&= \prod_{i=k+1}^{n}\frac{\alpha + \beta + 1+(2n - i -
  1)\gamma}{\alpha+1+(n-i)\gamma}.
\end{aligned}
\end{multline*}

So,
\begin{multline*}
  \left\langle x_1^2\cdots x_{n-1}^2\right\rangle_{n,\alpha,\beta,\gamma} 
   = \frac{\alpha+\beta+(n-1)\gamma+1}{\alpha(\alpha+1)}\times\frac{S_n(\alpha+2,\beta,\gamma)}{S_n(\alpha,\beta,\gamma)}\\
  \times\left(\alpha+\beta
  + \frac{\gamma(n-1)(\alpha+\beta+n\gamma+1)}{\alpha+\gamma
    +1}\right)
\end{multline*}
and,
\begin{multline*}
  \left\langle x_1^2\cdots
  x_{n-1}^2\right\rangle_{n,\alpha-2,\beta,\gamma} 
  = \frac{\alpha+\beta+(n-1)\gamma-1}{(\alpha-1)(\alpha-2)}
  \times\frac{S_n(\alpha,\beta,\gamma)}{S_n(\alpha-2,\beta,\gamma)}\\
  \times\left(\alpha+\beta-2 +
  \frac{\gamma(n-1)(\alpha+\beta+n\gamma-1)}{\alpha+\gamma
    -1}\right).
\end{multline*}
Hence
\begin{multline*}
\left\langle x_1^{-2}\right\rangle_{n,\alpha,\beta,\gamma} =
\frac{\alpha+\beta+(n-1)\gamma-1}{(\alpha-1)(\alpha-2)}\\
{}\times\left(\alpha+\beta - 2 +
\frac{\gamma(n-1)(\alpha+\beta+n\gamma-1)}{\alpha+\gamma -1}\right).
\end{multline*}
\end{proof}
Furthermore, by Propositions~\ref{prop:Selberg} and~\ref{prop:Aomoto},
\begin{lemma}\label{lem:twomoremoments}
  \begin{align*}
    \left\langle x_1^{-1}\right\rangle_{n,\alpha,\beta,\gamma} &=
    \frac{\alpha+\beta+(n-1)\gamma-1}{\alpha - 1},\hspace{\fill}\mbox{}
    \\
    \intertext{and}
    \left\langle
    x_1^{-1}x_2^{-1}\right\rangle_{n,\alpha,\beta,\gamma}
    &=
    \frac{(\alpha+\beta+(n-1)\gamma-1)(\alpha+\beta+n\gamma-1)}{(\alpha
      - 1)(\alpha + \gamma - 1)}.
  \end{align*}
\end{lemma}

Notice that Theorem~\ref{thm:Jacobimoments} and Lemma~\ref{lem:twomoremoments} above can be used to compute expected values of negative moments over the eigenvalues of the Jacobi Ensemble, where the conversion between $(\alpha,\beta)$ and $(m,n,p)$ is given in the description of the Jacobi distribution above in Proposition 3.

\subsection{A few simple identities}

Let us say that $M$ is any $a\times b$ matrix and $S$ is any $a\times a$ positive definite symmetric matrix, and $e$ and $f$ are vectors of i.i.d. standard Gaussians of lengths $a$ and $b$:

\begin{lemma}\label{lem:threeexpectations}
  \begin{equation*}
    \begin{split}
E[e^tSe] &= \trace{}(S), \\
E[(e^tSe)^2] &= \trace{}(S)^2 + 2\cdot\trace{}(S^2),\textrm{ and} \\
E[(e^tMf)(f^tM^te)] &= \trace{}(M^tM)
    \end{split}
\end{equation*}
\end{lemma}
\begin{proof}
\begin{align*}
E[e^tSe] = \sum_{i=1}^{a}\sum_{j=1}^{a}E[e_ie_j]S_{i,j}
= \sum_{i = 1}^{a}E[e_i^2]S_{i,i}
=\trace{}(S)
\end{align*}

\begin{align*}
E[(e^tSe)^2] &= \sum_{i=1}^{a}\sum_{j=1}^{a}\sum_{k=1}^{a}\sum_{l=1}^{a}E[e_ie_je_ke_l]S_{i,j}S_{k,l}\\
&= \sum_{i=1}^{a}E[e_i^4]S_{i,i}^2 + \sum_{i\neq k}E[e_i^2]E[e_k^2]S_{i,i}S_{k,k} + 2\sum_{i\neq j}E[e_i^2]E[e_j^2]S_{i,j}^2 \\
&= \sum_{i=1}^{a}3S_{i,i}^2 + \sum_{i\neq k}S_{i,i}S_{k,k} + 2\sum_{i\neq j}S_{i,j}^2 \\
&=  \sum_{i=1}^{a}\sum_{k=1}^{a}S_{i,i}S_{k,k} + 2\sum_{i=1}^{a}\sum_{j=1}^{a}S_{i,j}^2 \\
&= \trace{}(S)^2 + 2\cdot\trace{}(S^2)
\end{align*}

\begin{align*}
E[(e^tMf)(f^tM^te)] &= \sum_{i=1}^{a}\sum_{j=1}^{b}\sum_{k=1}^{a}\sum_{l=1}^{b}E[e_if_je_kf_l]M_{i,j}M_{k,l} \\
&= \sum_{i=1}^{a}\sum_{j=1}^{b}M_{i,j}^2 \\
&= \trace{}(M^tM)
\end{align*}
\end{proof}

\subsection{Primary results}\label{subsec:PrimaryResults}

We restate our main theorem.
\newcounter{savetheorem}
\setcounter{savetheorem}{\value{theorem}}
\newcounter{savesection}
\setcounter{savesection}{\value{section}}
\setcounter{theorem}{\value{vanillaregression}}
\setcounter{section}{\value{vanillaregressionsection}}
\begin{theorem}
    Let X and $\epsilon$ be i.i.d. standard Gaussians in
    $\mathbb{R}^{m\times n}$ and $\mathbb{R}^m$, let $b$ be in
    $\mathbb{R}^n$, let $\Sigma$ be positive definite in
    $\mathbb{R}^{n\times n}$ with Cholesky Decomposition $R^tR$, and let $y = XRb +
    \sigma\epsilon$, so the rows of $X$ have covariance $\Sigma$. Let $\hat{b} =
    \arg\min_{b^p}\|X_{1:p,:}Rb^p - y_{1:p}\|$ (all norms
    in this paper are Euclidean). Then
    \[
    \arg\min_{p}E\left[\left(\frac{1}{m-p}\|X_{p+1:m,:}R\hat{b}
      - y_{p+1:m}\|^2 - \sigma^2\right)^2\right]
    \]
    satisfies
\[
  \begin{split}\delta_{m,n}(p) &=
    9 + n (m^2 (2 + n)^2 + 3 (4 + n) - 2 m (5 + 2n)) \\
    &\qquad{}- 24 p -  n (12 + m^2 (2 + n) + 2 m n (3 + n)) p \\
  &\qquad{}+(22 + n (8 + 2 m (1 + n) + n (3 + n))) p^2 
  - (8 + n (2 + n)) p^3 + p^4 \\
  &=0,\end{split}
\]
as nearly as possible for an integer, giving $p = O(m^{2/3})$ with
leading order term $(2n+n^2)^{1/3}m^{2/3}$ for $m \gg n$.
\end{theorem}
\setcounter{theorem}{\value{savetheorem}}
\setcounter{section}{\value{savesection}}
\begin{proof}
  \[
  \hat{b} = (R^tX_{1:p,:}^tX_{1:p,:}R)^{-1}R^tX_{1:p,:}^ty_{1:p} =
  (X_{1:p,:}^tX_{1:p,:}R)^{-1}X_{1:p,:}^t(X_{1:p,:}Rb +
  \sigma\epsilon_{1:p})
  \]
and
\[ y_{p+1:m} = X_{p+1:m,:}Rb + \sigma\epsilon_{p+1:m} \]
So, plugging these expressions for $\hat{b}$ and $y_{p+1:m}$ into the
expression in Theorem~\ref{thm:vanillaregression}, the task is to find the $\arg\min_{p}$ of the expected value of:
\begin{multline*}\Bigg(\|X_{p+1:m,:}R(X_{1:p,:}^tX_{1:p,:}R)^{-1}X_{1:p,:}^t(X_{1:p,:}Rb +  \sigma\epsilon_{1:p})\\ - X_{p+1:m}Rb - \sigma\epsilon_{p+1:m}\|^2/(m - p) - \sigma^2\Bigg)^2. \end{multline*}
After cancellation within the norm (without expanding it), this task becomes finding
\[\arg\min_pE\left[\left(\frac{1}{m-p}\left\|\sigma A\epsilon_{1:p} - \sigma\epsilon_{p+1:m}\right\|^2 - \sigma^2\right)^2\right] \]
where $A = X_{p+1:m,:}\left(X_{1:p,:}^tX_{1:p,:}\right)^{-1}X_{1:p,:}^t$. Notice that $b$ and $R$ cancel out, $\sigma$ factors out, and they all become irrelevant.  Expanding, the next task is to find the $\arg\min_p$ of the expected value of
\begin{multline*}
  1 -\frac{2}{m-p}\big(\epsilon_{1:p}^tA^tA\epsilon_{1:p}
  + \epsilon_{p+1:m}^t\epsilon_{p+1:m} - 2\epsilon_{p+1:m}^tA\epsilon_{1:p}\big) \\
  + \frac{1}{(m-p)^2}\Big( (\epsilon_{1:p}^tA^tA\epsilon_{1:p})^2 +
  (\epsilon_{p+1:m}^t\epsilon_{p+1:m})^2
  +2(\epsilon_{1:p}^tA^tA\epsilon_{1:p}\epsilon_{p+1:m}^t\epsilon_{p+1:m}) \\
  \qquad\qquad{}
  +4(\epsilon_{p+1:m}^tA\epsilon_{1:p})^2
  -4 (\epsilon_{p+1:m}^t\epsilon_{p+1:m}\epsilon_{p+1:m}^tA\epsilon_{1:p})
  \\
  \qquad\qquad{}-4(\epsilon_{p+1:m}^tA\epsilon_{1:p}\epsilon_{1:p}^tA^tA\epsilon_{1:p})
  \Big).
\end{multline*}
Removing terms with expected value zero and replacing $E[\epsilon_{p+1:m}^t\epsilon_{p+1:m}]$ with $m-p$, this becomes:
\begin{multline*}
  1 -\frac{2}{m-p}\big(\epsilon_{1:p}^tA^tA\epsilon_{1:p}
  + m-p \big) \\
  + \frac{1}{(m-p)^2}\Big( (\epsilon_{1:p}^tA^tA\epsilon_{1:p})^2 +
  (\epsilon_{p+1:m}^t\epsilon_{p+1:m})^2
  +2(m-p)(\epsilon_{1:p}^tA^tA\epsilon_{1:p}) \\
  \qquad\qquad{}
  +4(\epsilon_{p+1:m}^tA\epsilon_{1:p})^2
  \Big).
\end{multline*}
Cancelling, this becomes
\begin{equation*}
 -1+\frac{1}{(m-p)^2}\Big( (\epsilon_{1:p}^tA^tA\epsilon_{1:p})^2 +
  (\epsilon_{p+1:m}^t\epsilon_{p+1:m})^2
  +4(\epsilon_{p+1:m}^tA\epsilon_{1:p})^2
  \Big).
\end{equation*}
Using Lemma~\ref{lem:threeexpectations}, the problem reduces to finding the $\arg\min_p$ of the expected value of:
\[ \frac{1}{(m-p)^2}\Big( 2\cdot\trace{}((A^tA)^2) +
  \trace{}(A^tA)^2+4\cdot\trace{}(A^tA) + 2(m-p) \Big).
\]
Note that since
\[
  A^tA =
  X_{1:p,:}(X_{1:p,:}^tX_{1:p,:})^{-1}X_{p+1:m,:}^t \cdot{}X_{p+1:m,:}(X_{1:p,:}^tX_{1:p,:})^{-1}X_{1:p,:}^t,
\]
if
\begin{align*}B&=
  X_{1:p,:}^tX_{1:p,:}(X_{1:p,:}^tX_{1:p,:})^{-1}X_{p+1:m,:}^t
   X_{p+1:m,:}(X_{1:p,:}^tX_{1:p,:})^{-1}\\
&=X_{p+1:m,:}^tX_{p+1:m,:}(X_{1:p,:}^tX_{1:p,:})^{-1},
\end{align*}
then $B$ is isospectral with $A^tA$ (at least as far as nonzero eigenvalues are concerned) and thus traces of their powers are the same. Hence, the problem reduces to finding the $\arg\min_p$ of the expected value of:
\[
  \frac{1}{(m-p)^2}\Big( 2\cdot\trace{}(B^2) + \trace{}(B)^2 
       +4\cdot\trace{}(B) + 2(m-p) \Big).
\]
Let $C$ be the Matrix Model for the Jacobi Ensemble as defined in
Proposition~\ref{prop:Jacobimodel}. $B = C^{-1} - I_{n\times n}$, with
$\alpha = \frac{1}{2}\left(p - n + 1\right)$, $\beta =
\frac{1}{2}\left(m - p - n + 1\right)$, and $\gamma = 1/2$. Let
$x_1,\ldots x_n$ be the eigenvalues of $C$. By
Proposition~\ref{prop:Jacobimodel}, next is to find the $\arg\min_p$
of the expected value of
\[
\frac{1}{(m-p)^2}\Bigg(2\sum_{i=1}^{n}\left(x_i^{-1} - 1\right)^2 +
\bigg(\sum_{i=1}^{n}\left(x_i^{-1} -
1\right)\bigg)^2+4\sum_{i=1}^{n}\left(x_i^{-1} - 1\right) + 2(m-p)
\Bigg)
\]
or rather
\[
\frac{1}{(m-p)^2}\Bigg(3\sum_{i=1}^{n}x_i^{-2} + \sum_{i\neq
  j}^{n}(x_ix_j)^{-1} -2n\sum_{i=1}^{n}x_i^{-1} + n^2 - 2n + 2(m - p)
\Bigg)
\]
which is $\arg\min_p$ of 
\begin{multline*}\frac{1}{(m-p)^2}\Bigg(3\sum_{i=1}^{n}\left\langle x_i^{-2}\right\rangle_{n,\alpha,\beta,\gamma}+ \sum_{i\neq j}^{n}\left\langle(x_i x_j)^{-1}\right\rangle_{n,\alpha,\beta,\gamma}\\-2n\sum_{i=1}^{n}\left\langle x_i^{-1}\right\rangle_{n,\alpha,\beta,\gamma} + n^2 - 2n + 2(m - p)  \Bigg), \end{multline*}
or
\begin{multline*}
  \frac{1}{(m-p)^2}\Bigg(3n\left\langle
  x_i^{-2}\right\rangle_{n,\alpha,\beta,\gamma} +
  n(n-1)\left\langle(x_i
  x_j)^{-1}\right\rangle_{n,\alpha,\beta,\gamma}\\
  -2n^2\left\langle x_i^{-1}\right\rangle_{n,\alpha,\beta,\gamma} +
  n^2 - 2n + 2(m - p)  \Bigg).
\end{multline*}
Using Theorem~\ref{thm:Jacobimoments} and Lemma~\ref{lem:twomoremoments} above, which can be used to compute expected values of negative moments over the Jacobi Ensemble, it remains to minimize over $p$:
\begin{multline}
  f(m,n)
  = \frac{3n(1 - m + n)(n(n + 3) + m(p - 1) - 2(n + 1)p)}{(n - p)(n -
    p + 1)(n - p + 3)(m - p)^2}
 + \frac{n^2 - 2n + 2(m - p)}{(m - p)^2}
  \\
  + \frac{n(n - 1)(-1 + m - n)(m - n)}{(n - p)(n - p + 1)(m - p)^2} 
- \frac{2n^2(1 - m + n)}{(n - p + 1)(m - p)^2}\label{eqn:tominimize}
\end{multline}
Differentiating with respect to $p$ and setting the formula equal to zero gives $\delta_{m,n}(p) = 0$.

The order of $p$ as a function of $m$ is found by the method of Dominant Balance. First consider $p = O(m^q)$ where $q \in (0,1]$, the $p = O(m^0)$ case is handled later (also, solving the equation $\delta_{m,n}(p) = 0$ analytically in Mathematica yields a complicated unsimplified expression with no logarthimic terms, this is possible because it is quartic in $m$). Assume $q = 1$. Then $\lim_{m\rightarrow\infty}\frac{\delta_{m,n}(p)}{m^4} = \frac{p^4}{m^4} = 0$, so the coefficient in the $O$-notation term for $O(m)$ would be zero, a contradiction. So for now it is assumed that $q\in(0,1)$. Then the only terms that matter in $\delta_{m,n}(p)$ are $p^4$ and $-m^2n(2+n)p$, taking $\lim_{m\rightarrow\infty}\frac{\delta_{m,n}(p)}{m^q}$ for various values of $q$ makes all the rest of the terms zero while preserving these as potentially nonzero. Setting the sum of these two terms equal to zero and ignoring two complex roots gives $p = O(m^{2/3})$ with leading order term $(2n+n^2)^{1/3}m^{2/3}$, which checks out since $p^4$ and $m^2n(2+n)p$ then have the same order, $8/3$.

Equation~\eqref{eqn:tominimize} is convex for $p > n + 3$. This means that any root of $O(m^0)$ of $\delta_{m,n}(p)$ would have to be a maximum of \eqref{eqn:tominimize}, since the other three roots are accounted for. Recall that the product of convex functions is convex, and note that \eqref{eqn:tominimize} simplifies to
\begin{equation*} f(m,n) = \frac{6 + m n (2 + n) - (8 + n (2 + n) - 2 p) p}{(m - p) (-3 - n + 
   p) (-1 - n + p)}. \end{equation*}
The numerator is obviously convex in $p$ since it is a quadratic and $p^2$ has a positive sign; it is necessary that one divided by the denominator is convex too.  $1/(m-p)$, $1/(-3-n+p)$, and $1/(-1-n+p)$ are all individually convex in $p$ for $p > n+3$ and $p < m$, so their product is convex.
\end{proof}

\begin{corollary}\label{cor:asymptotics} If $\delta_{m,n}(p^*(m,n)) = 0$ and $n$ is fixed,
\begin{multline*}
  p^*(m,n) 
  =  m^{2/3} (n (2 + n))^{1/3} -  m^{1/3}\cdot\frac{2 n (1 +
  n)}{3 (n (2 + n))^{1/3}}
  + \frac{1}{3}(6 + n + n^2) \\
  {}- m^{-1/3}\cdot\frac{2
  n^2 (216 + 230 n + 87 n^2 + 24 n^3 + 5 n^4)}{81 (n (2 + n))^{5/3}}+ \textrm{L.O.T.}
\end{multline*}
\end{corollary}
\begin{proof}
Proof by Mathematica. First, use it to solve the quartic to get $p^*$, then subtract some of the terms in this sequence from it and find the limit as $m\rightarrow\infty$ of the difference times an integer power of $m^{1/3}$.
\end{proof}

\begin{remark}
  The leading order term, $m^{2/3} (n (2 + n))^{1/3}$ in $p^*(m,n)$
  for fixed $n$ and large $m$ converges slowly, it is not within 1\%
  of the expression for $p^*(m,n)$ above without the Lower Order Terms
  until $m$ is twenty-seven million for n equal to ten. However, the
  expansion in Corollary~\ref{cor:asymptotics} converges quickly, as
  is seen in Fig.~\ref{fig:ratio5} and Fig.~\ref{fig:ratio20}.
\end{remark}

\subsection{Computational evidence}

\begin{figure}[tb!]
\centering\includegraphics[width=4in]{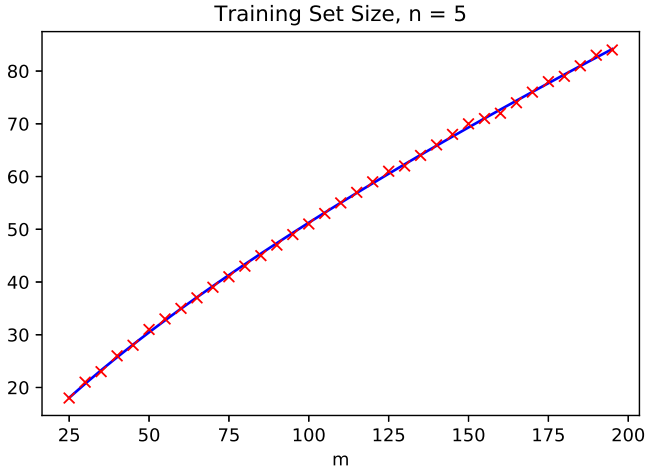}
\caption{\label{fig:optimalp5}Optimal $p$ by simulation and theorem, n
  = 5}
\end{figure}
  
  \begin{figure}[tb!]
\centering\includegraphics[width=4in]{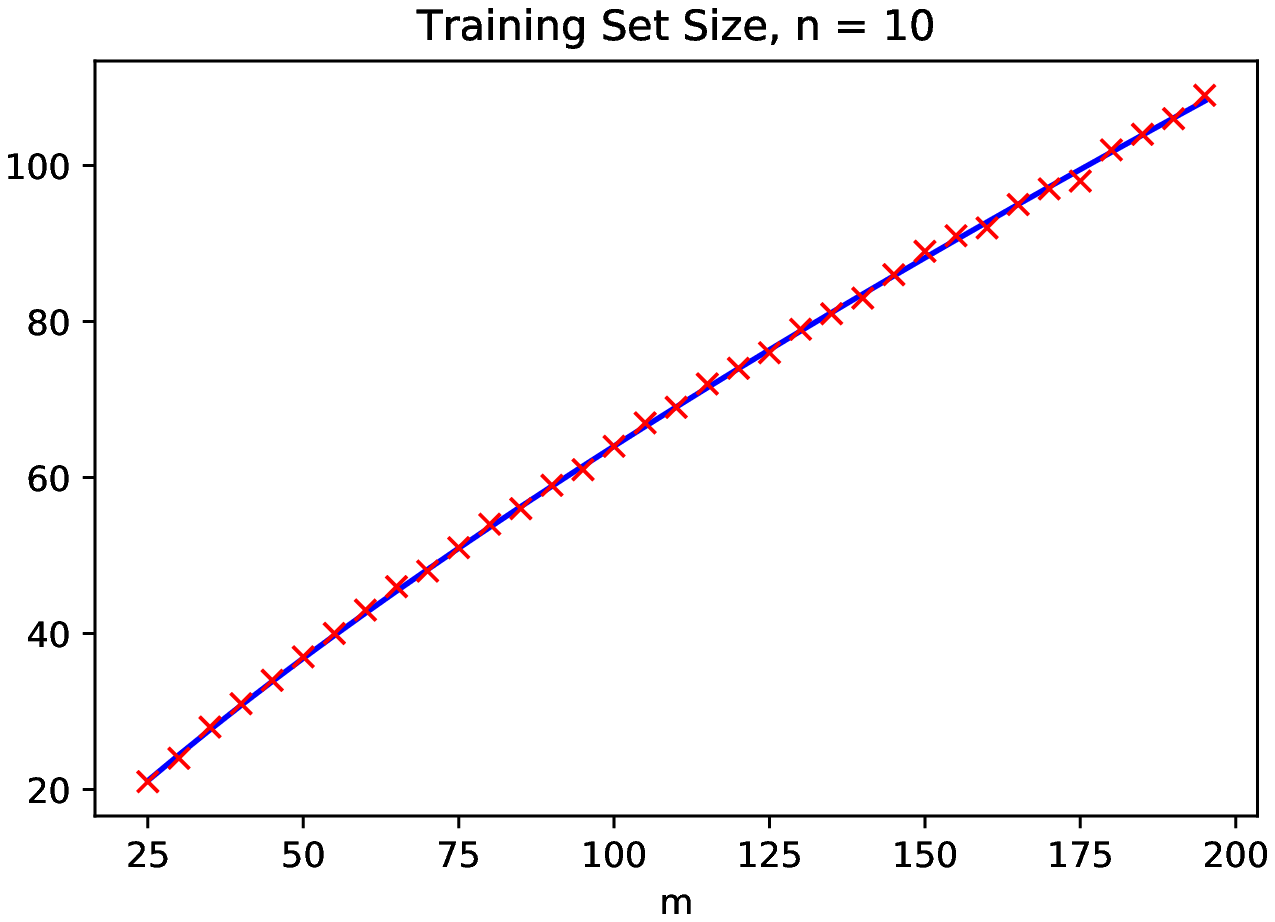}
\caption{\label{fig:optimalp10}Optimal $p$ by simulation and theorem,
  n = 10}
\end{figure}

\begin{figure}[tb!]
\centering\includegraphics[width=4in]{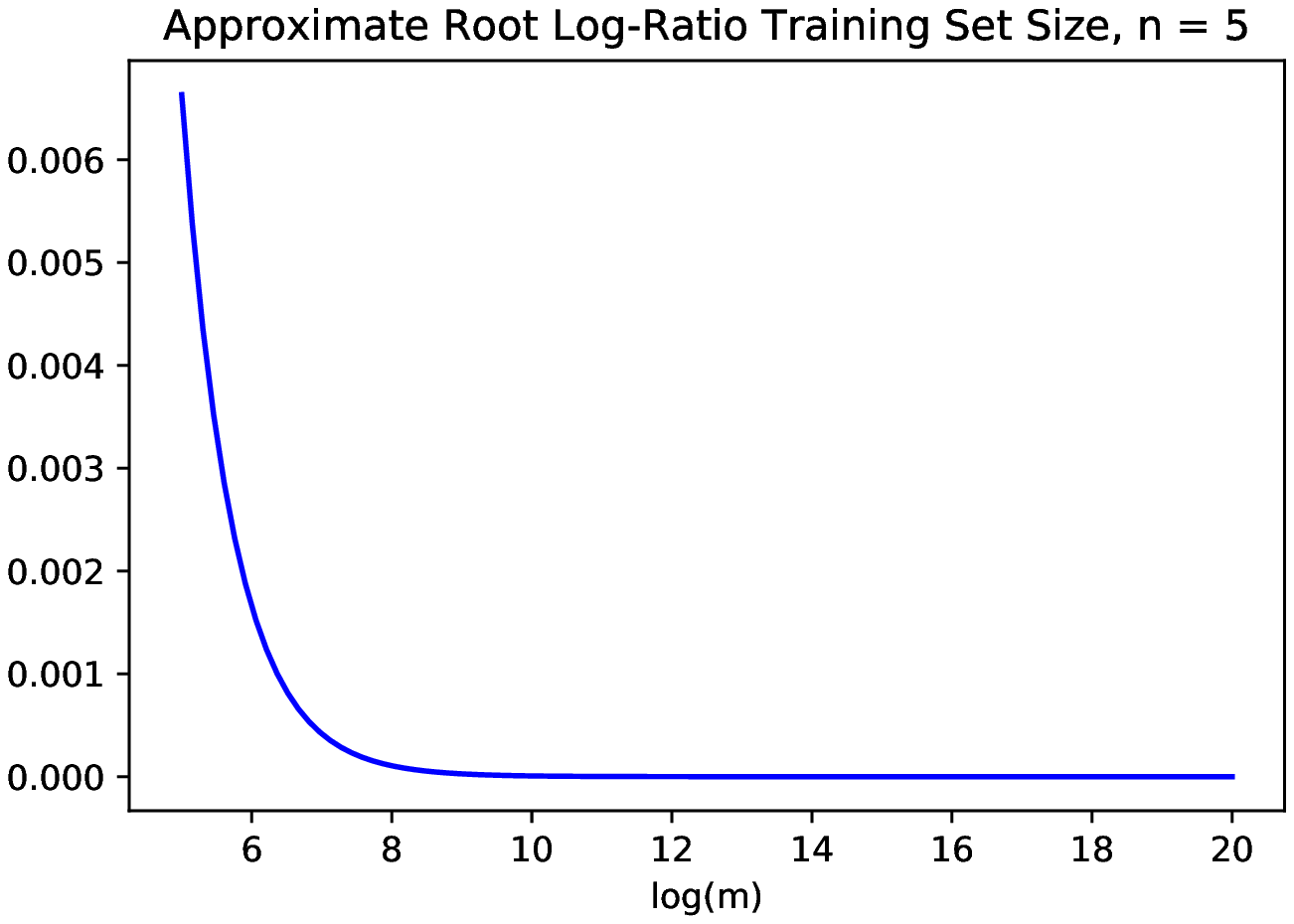}
\caption{\label{fig:ratio5}Ratio of optimal $p$ to Corollary~\ref{cor:asymptotics} expression, $n = 5$}
\end{figure}

\begin{figure}[tb!]
\centering\includegraphics[width=4in]{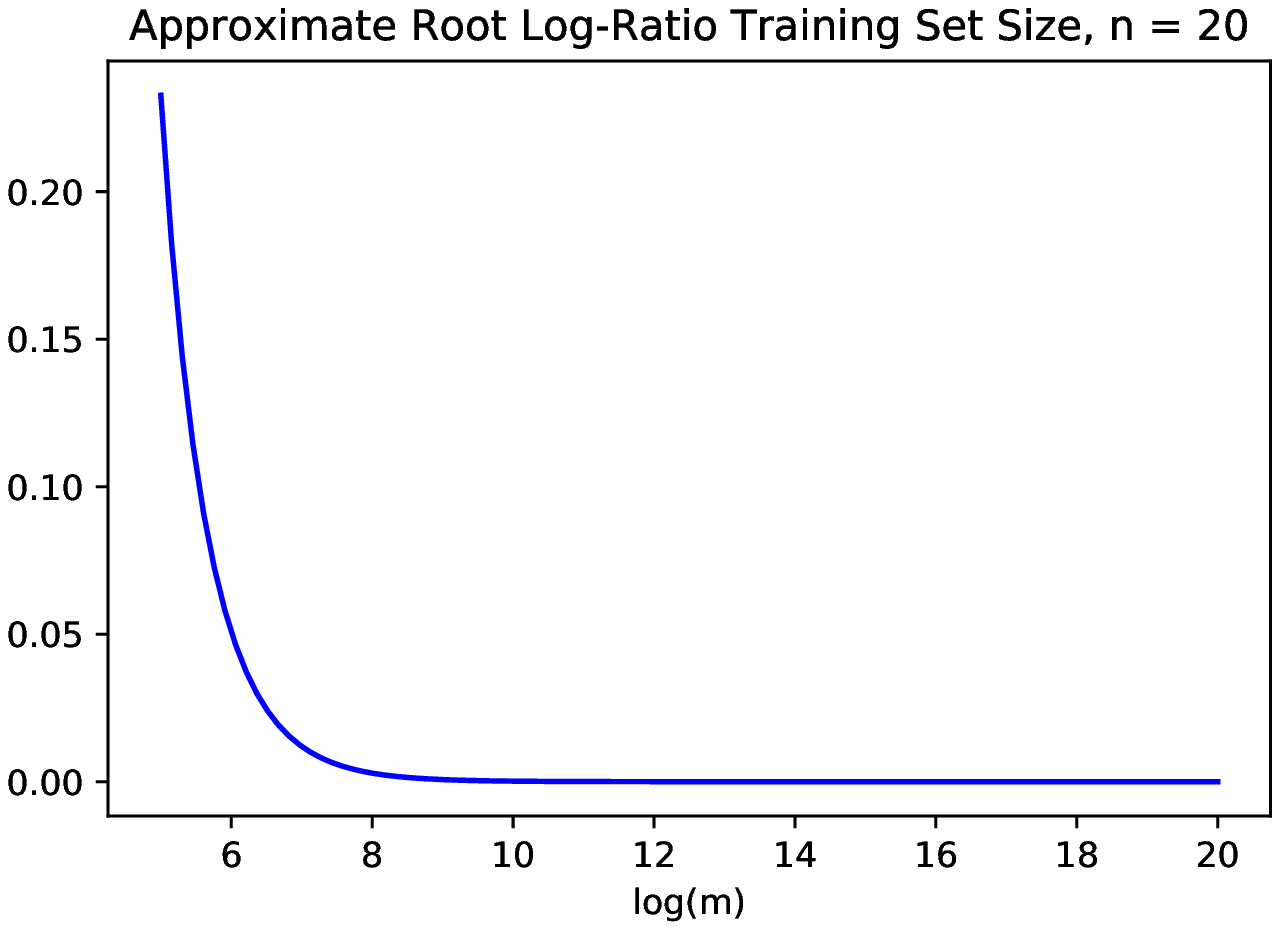}
\caption{\label{fig:ratio20}Ratio of optimal $p$ to Corollary~\ref{cor:asymptotics} expression, $n = 20$}
\end{figure}

The plots Fig.~\ref{fig:optimalp5} and Fig.~\ref{fig:optimalp10} show
the optimal $p^*(m,n)$, where $\delta_{m,n}(p^*(m,n)) = 0$, for a
given $m$ where $n = 5$ in Fig.~\ref{fig:optimalp5} (then $10$ in
Fig.~\ref{fig:optimalp10}) and $\sigma = 1$ in Fig.~\ref{fig:optimalp5}
(then $0.1$ in Fig.~\ref{fig:optimalp10}), by solving the quartic equation
which is $\delta_{m,n}(p) = 0$. Doing so for different values of $m$
creates a continuous curve, $p$ as a function of $m$, in
Fig.~\ref{fig:optimalp5} and Fig.~\ref{fig:optimalp10} The \verb|x|'s that
lie nearly on the curve are computed by simulation: First, for a given
$m$, $n$, and $\sigma$, pick $b$ and $\Sigma$ randomly. Then do the
following a million times: pick $X$ and $\epsilon$ as i.i.d standard
Gaussians in $\mathbb{R}^{m,n}$ and $\mathbb{R}^n$, then compute
\[y = XRb + \sigma\epsilon\]
and
\[\hat{b} = \arg\min_{b^p}\|X_{1:p,:}Rb^p - y_{1:p}\|\]
and
\[\left(\frac{1}{m-p}\|X_{p+1:m,:}R\hat{b} - y_{p+1:m}\|^2 - \sigma^2\right)^2\]
for every value of $p$ from $n + 1$ to $m - 1$. Enough data is now available to empirically compute the {\it integrity metric}
\[E\left[\left(\frac{1}{m-p}\|X_{p+1:m,:}R\hat{b} - y_{p+1:m}\|^2 - \sigma^2\right)^2\right]\]
for every value of $p$. Find the value of $p$ that minimizes it, and plot it as a \verb |x|  over its corresponding value of $m$ ($n$ is fixed for any given figure), which is near the curve described above. Note that the figures are approximately lines. These simulations verify Theorem~\ref{thm:vanillaregression}.

The above plots Fig.~\ref{fig:ratio5} and Fig.~\ref{fig:ratio20} show
the log-ratio of $p$ computed by $\delta_{m,n}(p) = 0$ and $p$
approximated by Corollary~\ref{cor:asymptotics} with the same parameters as before. Convergence is fast.

\subsection{Analysis of the method on real data}

\begin{table}
  \caption{\label{empiricalloss}This table shows the empirical losses on several data sets for given training-test divisions. The last two rows show the ratio used for training and the size of the data set.}
\centering\scalebox{.9}{\begin{tabular}{|c|c|c|c|c|c|c|}
\hline
& \cite{Coraddu2016}  &  \cite{Cho2020} & \cite{Chicco2020} & \cite{Abid2020} & \cite{Rafiei2015} & \cite{Santos2015} \\
\hline
\vphantom{$\frac{1^2}{2_3}$}$p = \frac{1}{2}m$&3.648e-06 & 2.149&0.1455 &0.08236 &213618 &1.311\\
\hline
\vphantom{$\frac{1^2}{2_3}$}$p = \frac{3}{4}m$& 3.645e-06& 2.144&0.1409 &0.07821 &78656 &0.3778\\
\hline
$p = p^*(m,n)$ &3.661e-06 & 2.156&0.1445 &0.08178 &55066 &0.2911\\
\hline
$p^*(m,n)/m$& 0.2434&0.3154 &0.5385 & 0.5267 &0.8871 &0.8545\\
\hline
$(m,n)$& (11934,16)&(7590, 21) & (299, 12)& (243, 10) & (372, 107)&(165, 49)\\
\hline
\end{tabular}}
\end{table}

Given detrended real data $X$ and labels $y$ (in both the averages of each columns are removed), the optimal $p$ and the loss for dividing the data into $X_{1:p,:}$ and $X_{p+1:m,:}$ for training and testing are found. Also found are the loss for certain standard values of $p$, $p = \frac{1}{2}m$ and $p = \frac{3}{4}m$. It is calculated by finding $\hat{b}$ where:
\[ \hat{b} = \arg\min_{b^p}\|X_{1:p,:}b^p - y_{1:p}\| \]
and then computing the ``loss'':
\[ \frac{1}{m - p}\|X_{p+1:m,:}\hat{b} - y_{p+1:m}\|^2. \]
for ten thousand permutations of the rows of $X$ and $y$, and then taking the average. If there are more than one variables to predict, $y$ is set to be the first one the others are ignored. Also ignored are columns with date information and rows with NaNs, columns with missing values indicated by ``?'' are interpolated as the average value of the column. $X$ and $y$ come from the data sets \cite{Coraddu2016}, \cite{Cho2020}, \cite{Chicco2020}, \cite{Abid2020}, \cite{Rafiei2015}, and \cite{Santos2015}, which are all available at {\it https://archive.ics.uci.edu/ml/datasets.php}. Table~\ref{empiricalloss}., computes the average loss on each data set over ten thousand permutations of the rows of $X$ and $y$ for $p = \frac{1}{2}m$, $p = \frac{3}{4}m$, and $p = p^*(m,n)$ as defined above, and also gives $p^*(m,n)/m$.

First observe that the training set sizes are very small in the big data case in the first two column, very large in the small data case ($m$ near $n$) in the last two columns, and in the realm that one would expect in the case in between, in the third and fourth columns. In the small data case the loss can be way off if $p$ is chosen poorly.

\section{Discussion}
There may soon come a time when machine learning practitioners, faced with $m$ data points of dimension $n$, will plug $(m,n)$ into an algorithm that suggests a good size for the training, validation, and test sets for their models in a matter of seconds, instead of relying on gut and hearsay. A natural generalization to this paper's results would be for linear models with more complex features. It would be interesting if $O(m^{2/3})$ sized training sets were a general law for large datasets.


\end{document}